\documentclass[11pt]{article}

\usepackage[margin=1in]{geometry}
\usepackage{amsmath,amssymb,amsthm,mathtools}
\usepackage{microtype}
\usepackage[hidelinks]{hyperref}
\allowdisplaybreaks

\hypersetup{
  pdftitle={Second-Moment Memory in Coordinatewise Adam},
  pdfauthor={Jeonseong Kim},
  pdfsubject={Second-moment memory and positive normalized updates in Adam},
  pdfkeywords={Adam, second-moment memory, finite variance, lower bound}
}

\newtheorem{theorem}{Theorem}
\newtheorem{lemma}{Lemma}
\newtheorem{proposition}{Proposition}
\newtheorem{corollary}{Corollary}
\newtheorem{remark}{Remark}

\newcommand{\R}{\mathbb R}
\newcommand{\E}{\mathbb E}
\newcommand{\Prb}{\mathbb P}
\newcommand{\sgn}{\operatorname{sign}}
\newcommand{\pos}[1]{\left(#1\right)_+}
\newcommand{\abs}[1]{\left\lvert #1\right\rvert}
\newcommand{\eps}{\varepsilon}

\title{\textbf{Second-Moment Memory in Coordinatewise Adam}}
\author{Jeonseong Kim}
\date{}

\begin{document}
\maketitle

\begin{abstract}
Adam retains a moving average of past squared gradients in its denominator,
but the optimization cost of this memory is not well understood.  We show
that second-moment memory can itself suppress progress toward the optimum even
under finite-variance stochastic gradients.  For a simple two-point oracle,
the expected positive normalized update is \(O(M_2^{-1/2})\) after an
initialization transient, where
\(M_2=(1-\beta_2)^{-1}\) is the second-moment memory length.  We convert this
directional bound, under the stated memory and stepsize scaling, into an
average-stationarity lower bound of the same order on a smooth convex problem
with normalized gap, smoothness, and variance.  Long second-moment memory can
slow optimization even when the gradient noise has finite variance.
\end{abstract}

\section{Introduction}

Adam divides a first-moment exponential moving average by the square root of
a second-moment exponential moving average
\cite{kingma2015adam}.  Past squared gradients
continue to affect the denominator after the gradients themselves are no
longer current.  We study how this retained history affects optimization.

The finite-\(p\) model permits increasingly severe tails as \(p\) decreases
below two, but a lower-bound construction need not use the full severity of
this class.  We work at the finite-variance endpoint \(p=2\); the oracle
constructed below also satisfies the standard centered \(p\)-th moment bound
for every \(1<p\le2\).  At this endpoint the
rare-event scale is critical: an event of probability \(\rho\) and magnitude
\(\Gamma=\Theta(\rho^{-1/2})\) contributes only constant-order mass to the
second-moment update, \(\rho \Gamma^2=\Theta(1)\).  Nothing in the second moment
itself diverges.  Taking \(\rho=1-\beta_2\) also matches the event recurrence
time \(\rho^{-1}\) to the forgetting time \((1-\beta_2)^{-1}\) of Adam's
second-moment state.

A rare large gradient appears only once, but its square remains in the
second-moment accumulator \(v_t\) for many subsequent steps.  This keeps the
denominator large and suppresses ordinary updates.  Repeated rare events can
refresh the stored square before its contribution has decayed.  The slowdown
comes from past squared gradients remaining in \(v_t\), not from infinite
variance.

We quantify this effect using a finite-variance two-point
oracle.  Our first result shows that, in the low-signal regime and after an
initialization transient, the expected positive normalized update is
\(O(M_2^{-1/2})\), where \(M_2=(1-\beta_2)^{-1}\) is the second-moment memory
length.  Longer memory yields a precise finite-time upper bound on
progress toward the optimum.

Our second result converts this directional bound into an optimization lower
bound on a smooth convex one-dimensional problem with
\(\Delta=L=\sigma=1\).  For \(M_2^{3/2}\le T\) and steps of order at most
\(M_2/T\), Adam has average stationarity at least
\(\frac12M_2^{-1/2}\).  A memoryless comparison has average stationarity
\(O_\kappa(M_2^{-1})\) on the same instance.  The two updates differ only
through the accumulated second-moment state.  Only at the end do we set
\(M_2=T^{2s}\), obtaining the horizon-dependent rates \(T^{-s}\) and
\(T^{-2s}\) for fixed \(s\in(0,1/3]\).

\paragraph{Related work.}
AdamW decouples weight decay from the adaptive update
\cite{loshchilov2019adamw}.
Classical counterexamples show that Adam-type exponential normalization can
fail even on simple convex problems \cite{reddi2018convergence}.  A
complementary line establishes convergence of Adam and related adaptive
methods under different stochastic assumptions and parameter regimes
\cite{chen2019convergence,zhang2022adam,wang2023closing,jin2025framework,
li2025adamw}.  In particular, recent analyses under coordinatewise bounded
variance obtain average \(\ell_1\)-stationarity guarantees for RMSProp and AdamW
\cite{li2025rmsprop,li2025adamw}.

Several works study the internal mechanisms created by adaptive normalization.
Balles and Hennig separate the sign and magnitude components of Adam updates
\cite{balles2018dissecting}, while AdaShift and ADOPT modify the temporal
interaction between the current gradient and second-moment normalization
\cite{zhou2019adashift,taniguchi2024adopt}.  More recently, Yu et al.\ formulate
AdamW under finite-\(p\) noise as an open problem and give a corridor
lower-bound mechanism in which denominator memory can suppress progress
\cite{yu2026open}.  We bound how long stored squared gradients continue to
affect the denominator and the resulting positive normalized update.
Recent work has also
developed finite-\(p\) theory for sign-based methods and Adam
\cite{yu2026sign,pang2026adam}.

\section{Effect of second-moment memory}

\subsection{Adam and the finite-variance oracle}

Let \(x_t\in\R^d\), with \(m_0=v_0=0\).  At call \(t\), the algorithm receives
a stochastic gradient \(g_t\) and updates
\begin{align}
 m_t&=\beta_1m_{t-1}+(1-\beta_1)g_t,
 \label{eq:adam-m}\\
 v_t&=\beta_2v_{t-1}+(1-\beta_2)g_t\odot g_t,
 \label{eq:adam-v}\\
 x_{t+1}&=x_t-\eta_t\frac{m_t}{\sqrt{v_t}+\eps_t}.
 \label{eq:adam-x}
\end{align}
Square roots and quotients are coordinatewise.  Since Adam acts coordinatewise,
we henceforth analyze a single coordinate and take \(d=1\).  We use batch size
one, zero weight decay, fixed \(\beta_1,\beta_2\), predictable nonnegative
steps \(\eta_t\), and predictable nonnegative denominator offsets \(\eps_t\).

Fix \(\sigma>0\), \(0\le\beta_1<1\), and \(1/2\le\beta_2<1\).  Write
\[
 \rho=1-\beta_2,
 \qquad
 M_2=\rho^{-1},
 \qquad
 \Gamma=\sigma\sqrt{\frac{\beta_2}{\rho}}.
\]
The subscript in \(M_2\) refers to the second-moment recursion \(v_t\).  Its
geometric weights decay on the timescale \((1-\beta_2)^{-1}\), which we call
the second-moment memory length.

Let \(\xi_t\) be i.i.d.\ with
\begin{equation}
 \xi_t=
 \begin{cases}
  -\Gamma,&\text{with probability }\rho,\\[1mm]
  \dfrac{\rho \Gamma}{\beta_2},&\text{with probability }\beta_2.
 \end{cases}
 \label{eq:fv-noise}
\end{equation}
Direct calculation gives
\begin{equation}
 \E\xi_t=0,
 \qquad
 \E\xi_t^2=\sigma^2.
 \label{eq:fv-moments}
\end{equation}
At a constant population gradient \(\mu>0\), set \(g_t=\mu+\xi_t\).

\subsection{Scale calculation}

The scale of the positive update is already visible in the two atoms.  The rare
outlier writes constant-order mass into the second moment, whereas
unbiasedness forces the common positive compensation to have only
square-root scale:
\begin{equation}
 \rho \Gamma^2
 =\sigma^2\beta_2\asymp\sigma^2,
 \qquad
 \frac{\rho \Gamma}{\beta_2}
 =\sigma\sqrt{\frac{\rho}{\beta_2}}
 \asymp\sigma\sqrt\rho.
 \label{eq:atom-scales}
\end{equation}
In the low-signal regime \(\mu/\sigma\lesssim\sqrt\rho\), the numerator's
positive scale is \(\sqrt\rho\), while a remembered outlier keeps the
denominator at constant order.  Since both the outlier recurrence time and
the memory lifetime are \(\rho^{-1}=M_2\), this scaling gives
\begin{equation}
 \boxed{\textnormal{predicted positive-update scale}
 \asymp\sqrt\rho=M_2^{-1/2}.}
 \label{eq:memory-heuristic}
\end{equation}

\begin{theorem}[Positive normalized update bound]
\label{thm:positive-update}
Assume \(\Gamma\ge2\mu\).  Under the oracle \eqref{eq:fv-noise}, let
\(g_t=\mu+\xi_t\), let \(m_t,v_t\) follow
\eqref{eq:adam-m}--\eqref{eq:adam-v}, let \(\eps_t\ge0\) be any
predictable denominator offset, and define
\[
 u_t=\frac{m_t}{\sqrt{v_t}+\eps_t}.
\]
For every \(t\ge1\),
\begin{equation}
\begin{aligned}
 \E\pos{u_t}
 &\le
 (1-\beta_1^t)
 \left[
 A_{\beta_2}
 \left(
 \frac{\mu}{\sigma}+\sqrt{\frac{\rho}{\beta_2}}
 \right)
 (1-\beta_2^{t/2})
 +
 \frac{\beta_2^t}{\sqrt{1-\beta_2^t}}
 \right],\\
 A_{\beta_2}
 &=\frac{2(1+\sqrt{\beta_2})}{\sqrt{\beta_2}}.
\end{aligned}
\label{eq:positive-update}
\end{equation}
\end{theorem}

The theorem turns the calculation in \eqref{eq:atom-scales} into a
finite-time upper bound by conditioning on the time since the most recent
outlier.  One such outlier suffices because \(v_t\) is a nonnegative sum of
historical squared gradients, so any surviving outlier certifies a large
denominator.

We focus on \((u_t)_+\) because only positive normalized directions move the
iterate toward the optimum; negative directions cannot help it escape the
corridor used below.

\subsection{Proof of Theorem~\ref{thm:positive-update}}

\begin{proof}[Proof of Theorem~\ref{thm:positive-update}]
Let
\[
 g_-=\mu-\Gamma<0,
 \qquad
 g_+=\mu+\frac{\rho \Gamma}{\beta_2}>0.
\]
Partition the first \(t\) draws by the latest negative outlier.  For
\(s=0,\ldots,t-1\), let \(E_s\) be the event that the most recent outlier occurs
at time \(t-s\).  Then
\begin{equation}
 \Prb(E_s)=\rho\beta_2^s.
 \label{eq:age-law}
\end{equation}
On \(E_s\), the stored square of that outlier gives
\begin{equation}
 v_t
 \ge\rho\beta_2^s(\Gamma-\mu)^2
 \ge\frac14\rho\beta_2^s\Gamma^2
 =\frac14\sigma^2\beta_2^{s+1}.
 \label{eq:outlier-v}
\end{equation}
All first-moment weights are nonnegative and sum to \(1-\beta_1^t\).  Since
every observation is at most \(g_+\),
\[
 (m_t)_+\le(1-\beta_1^t)g_+.
\]
The offset decreases the positive direction.  Hence
\[
 (u_t)_+\mathbf 1_{E_s}
 \le
 \frac{2(1-\beta_1^t)g_+}{\sigma}
 \beta_2^{-(s+1)/2}\mathbf 1_{E_s}.
\]
Taking expectations and summing over the outlier age yields
\begin{align}
 \sum_{s=0}^{t-1}\E[(u_t)_+\mathbf 1_{E_s}]
 &\le
 \frac{2(1-\beta_1^t)g_+}{\sigma\sqrt{\beta_2}}
 \rho\sum_{s=0}^{t-1}\beta_2^{s/2}\notag\\
 &\le
 (1-\beta_1^t)A_{\beta_2}
 \left(\frac\mu\sigma+\sqrt{\frac\rho{\beta_2}}\right)
 (1-\beta_2^{t/2}).
 \label{eq:age-sum}
\end{align}
The remaining event contains no outlier among the first \(t\) draws and has
probability \(\beta_2^t\).  On this event,
\[
 m_t=(1-\beta_1^t)g_+,
 \qquad
 v_t=(1-\beta_2^t)g_+^2.
\]
Its contribution is at most
\[
 \beta_2^t\frac{1-\beta_1^t}{\sqrt{1-\beta_2^t}}.
\]
Adding the two contributions proves \eqref{eq:positive-update}.
\end{proof}

\subsection{Initialization transient}

\begin{lemma}[Initialization transient bound]
\label{lem:transient}
For every \(q\in(0,1)\),
\begin{equation}
 \sum_{t=1}^{\infty}\frac{q^t}{\sqrt{1-q^t}}
 \le\frac{2}{-\log q}
 \le\frac{2}{1-q}.
 \label{eq:transient}
\end{equation}
\end{lemma}

\begin{proof}
The function \(h(x)=q^x/\sqrt{1-q^x}\) is positive and decreasing on
\((0,\infty)\).  The integral test and the substitution \(y=q^x\) give
\[
 \sum_{t=1}^{\infty}h(t)
 \le\int_0^\infty h(x)\,dx
 =\frac2{-\log q}.
\]
The inequality \(-\log q\ge1-q\) completes the proof.
\end{proof}

\begin{corollary}[Memory-length form]
\label{cor:memory-length}
Assume \(M_2\ge2\), \(\Gamma\ge2\mu\), \(\mu/\sigma\le M_2^{-1/2}\), and
\(\beta_1=0\).  Then, for every \(t\ge1\),
\begin{equation}
 \E\pos{u_t}
 \le
 12M_2^{-1/2}+
 \frac{(1-M_2^{-1})^t}{\sqrt{1-(1-M_2^{-1})^t}},
 \label{eq:memory-length}
\end{equation}
and the transient satisfies
\begin{equation}
 \sum_{t=1}^{\infty}
 \frac{(1-M_2^{-1})^t}{\sqrt{1-(1-M_2^{-1})^t}}
 \le2M_2.
 \label{eq:memory-transient-mass}
\end{equation}
The steady expected positive normalized update is \(O(M_2^{-1/2})\),
while initialization contributes at most \(2M_2\) in total.
\end{corollary}

\begin{proof}
For \(\beta_2\ge1/2\),
\[
 A_{\beta_2}\le2(1+\sqrt2),
 \qquad
 \sqrt{\rho/\beta_2}\le\sqrt{2\rho}.
\]
Together with \(\mu/\sigma\le\sqrt\rho\), the steady coefficient in
Theorem~\ref{thm:positive-update} is below \(12\sqrt\rho\).  Equation
\eqref{eq:memory-length} follows from \(\rho=M_2^{-1}\), and
\eqref{eq:memory-transient-mass} follows from
Lemma~\ref{lem:transient}.
\end{proof}

\begin{remark}
With equal memories \(\beta_1=\beta_2=1-M_2^{-1}\), the same argument gives
\(\E(u_t)_+\le12M_2^{-1/2}+(1-M_2^{-1})^t\).  The optimization construction
uses \(\beta_1=0\) to isolate the accumulated second moment.
\end{remark}

\section{From the directional bound to optimization}

\subsection{Hard instance}

The following function is the simplest smooth embedding of the
constant-gradient setting: it preserves derivative \(\mu\) on a long interval
and then connects smoothly to a flat optimum.

For \(\mu,L,W>0\), define
\begin{equation}
 f_{\mu,L}(x)=
 \begin{cases}
  0,&x\le-\mu/L,\\[1mm]
  \dfrac L2(x+\mu/L)^2,&-\mu/L\le x\le0,\\[2mm]
  \mu x+\dfrac{\mu^2}{2L},&x\ge0.
 \end{cases}
 \label{eq:corridor}
\end{equation}
It is convex, lower bounded, and \(C^{1,1}\) with smoothness \(L\).  Start at
\(x_1=W\).  The initial gap is
\begin{equation}
 \Delta=\mu W+\frac{\mu^2}{2L}.
 \label{eq:gap}
\end{equation}
Use the following i.i.d.\ stochastic gradient oracle:
\begin{equation}
 G(x,\omega)=f_{\mu,L}'(x)+\xi(\omega)
 =\frac{\partial}{\partial x}
 \left(f_{\mu,L}(x)+\xi(\omega)x\right).
 \label{eq:corridor-oracle}
\end{equation}

\subsection{Memory-length scaling}

The directional bound suggests the signal level \(\mu=M_2^{-1/2}\).  Keeping
the initial gap equal to one then sets the initial distance \(W\) at order
\(M_2^{1/2}\), while a horizon-\(T\) step should have scale \(M_2/T\):
\begin{equation}
 \rho=M_2^{-1},\quad
 \beta_1=0,\quad
 \beta_2=1-M_2^{-1},\quad
 \sigma=L=1,\quad
 \mu=M_2^{-1/2},\quad
 W=M_2^{1/2}-\frac12M_2^{-1/2}.
 \label{eq:scaling}
\end{equation}
The leading displacement budget matches this distance scale:
\begin{equation}
 T
 \cdot
 \frac{M_2}{T}
 \cdot
 M_2^{-1/2}
 =M_2^{1/2}\asymp W,
 \qquad
 \mu W+\frac{\mu^2}{2}=1.
 \label{eq:balance}
\end{equation}
The initialization transient contributes displacement of order \(M_2^2/T\).
The condition \(M_2^{3/2}\le T\) is exactly what keeps this term at most of
order \(M_2^{1/2}\).

\begin{theorem}[Optimization lower bound for Adam]
\label{thm:optimization}
Fix \(\kappa\in(0,1/64]\).  Let \(T\in\mathbb N\) and \(M_2\ge3\) satisfy
\(M_2^{3/2}\le T\).  Run Adam on
\eqref{eq:corridor}--\eqref{eq:corridor-oracle} from \(x_1=W\) with
the parameters in \eqref{eq:scaling}.  Then \(\Delta=L=\sigma=1\).  For
every predictable denominator offset \(\eps_t\ge0\) and every predictable
learning-rate sequence satisfying
\begin{equation}
 0\le\eta_t\le\kappa\frac{M_2}{T},
 \label{eq:step-envelope}
\end{equation}
the Adam iterates obey
\begin{equation}
 \Prb(x_t\ge0\text{ for every }1\le t\le T)\ge\frac12,
 \qquad
 \frac1T\sum_{t=1}^T\E\abs{f_{\mu,1}'(x_t)}
 \ge\frac12M_2^{-1/2}.
 \label{eq:optimization-lower}
\end{equation}
\end{theorem}

\subsection{Crossing bound}

On the half-line \(x\ge0\), the objective has constant derivative
\(f_{\mu,L}'(x)=\mu\).  Remaining in this region for a constant fraction of the
horizon forces average stationarity of order \(\mu\).

\begin{lemma}[Crossing bound]
\label{lem:crossing}
Run Adam on \eqref{eq:corridor}--\eqref{eq:corridor-oracle} from
\(x_1=W\).  Let \(\widetilde u_t\) be the normalized direction of a reference
process that always receives \(\widetilde g_t=\mu+\xi_t\).  Suppose
\(\E(\widetilde u_t)_+\le B_t\) and
\(0\le\eta_t\le\bar\eta_t\), where \(\bar\eta_t\) is deterministic.  Then
\begin{equation}
 \Prb(x_t\ge0\text{ for every }1\le t\le T)
 \ge
 1-\frac1W\sum_{t=1}^{T-1}\bar\eta_tB_t.
 \label{eq:crossing}
\end{equation}
\end{lemma}

\begin{proof}
The actual and reference processes agree until the first query below zero.
Let \(\tau=\inf\{t\ge1:x_t<0\}\).  If \(\tau\le T\), the updates before
\(\tau\) have crossed a distance larger than \(W\).  Negative normalized
directions move to the right, so
\[
 \{\tau\le T\}
 \subseteq
 \left\{
 \sum_{t=1}^{T-1}\bar\eta_t(\widetilde u_t)_+>W
 \right\}.
\]
Markov's inequality proves \eqref{eq:crossing}.
\end{proof}

\subsection{Proof of the optimization theorem}

\begin{proof}[Proof of Theorem~\ref{thm:optimization}]
Here \(\sqrt\rho=\mu=M_2^{-1/2}\).  Since \(M_2\ge3\),
\[
 \beta_2\ge\frac12,
 \qquad
 \Gamma=\sqrt{M_2-1}\ge2M_2^{-1/2}=2\mu.
\]
Corollary~\ref{cor:memory-length} gives
\begin{equation}
 \sum_{t=1}^{T-1}B_t
 \le12T M_2^{-1/2}+2M_2.
 \label{eq:B-sum}
\end{equation}
For \(\bar\eta_t=\kappa M_2/T\),
\begin{align*}
 \sum_{t=1}^{T-1}\bar\eta_tB_t
 &\le
 \kappa\frac{M_2}{T}
 \left(12T M_2^{-1/2}+2M_2\right)\\
 &=12\kappa M_2^{1/2}+2\kappa\frac{M_2^2}{T}\\
 &\le14\kappa M_2^{1/2}
 \le\frac14M_2^{1/2}
 \le\frac W2.
\end{align*}
The third line uses \(M_2^{3/2}\le T\), \(\kappa\le1/64\), and
\(W\ge\frac12M_2^{1/2}\).  Lemma~\ref{lem:crossing} gives survival
probability at least \(1/2\).  On the survival event,
\(\abs{f_{\mu,1}'(x_t)}=\mu\) for every \(t\le T\), which proves
\eqref{eq:optimization-lower}.  Finally,
\[
 \Delta
 =M_2^{-1/2}
 \left(M_2^{1/2}-\frac12M_2^{-1/2}\right)
 +\frac1{2M_2}
 =1.
\]
\end{proof}

\section{Memoryless comparison}

With \(\beta_1=0\), Adam's numerator is the current gradient.  Replacing the
accumulated denominator \(\sqrt{v_t}\) by the current magnitude \(\abs{g_t}\)
removes the accumulated second-moment state and yields exactly the raw signSGD
update \cite{bernstein2018signsgd}:
\begin{equation}
 y_{t+1}=y_t-\eta\sgn G(y_t,\omega_t),
 \qquad y_1=W.
 \label{eq:sign-update}
\end{equation}
The sign of each oracle observation is independent of the query point because
\(0\le f_{\mu,1}'(x)\le\mu\) and \(\Gamma\ge2\mu\).  The comparison
trajectory is a biased random walk.

\begin{proposition}[Memoryless comparison]
\label{prop:memoryless-comparison}
Fix \(\kappa\in(0,1/64]\).  There exists \(M_{2,0}(\kappa)\) such that, for
every instance in Theorem~\ref{thm:optimization} with
\(M_2\ge M_{2,0}(\kappa)\), setting \(\eps_t=0\) and taking the
constant step \(\eta=\kappa M_2/T\) for both Adam and
\eqref{eq:sign-update} gives
\begin{equation}
 \frac1T\sum_{t=1}^T\E\abs{f_{\mu,1}'(y_t)}
 \le\frac5\kappa M_2^{-1}.
 \label{eq:sign-rate}
\end{equation}
\end{proposition}

The scale follows directly from the crossing time.  A constant-order positive
direction reaches the flat region in \(O_\kappa(T/M_2^{1/2})\) steps, so
\begin{equation}
 \frac{T/M_2^{1/2}}{T}
 \cdot
 M_2^{-1/2}
 =M_2^{-1}.
 \label{eq:comparison-balance}
\end{equation}

\begin{proof}
Let \(X_t=\sgn G(y_t,\omega_t)\).  Then \(X_t\) are i.i.d.\ signs with
\[
 \Prb(X_t=1)=1-M_2^{-1},
 \qquad
 \E X_t=1-2M_2^{-1}.
\]
Define
\[
 S_n=\sum_{t=1}^nX_t,
 \qquad
 A_T=\frac{W+\mu}{\eta},
 \qquad
 n_T=\left\lceil\frac{4T}{\kappa M_2^{1/2}}\right\rceil.
\]
The iterate satisfies
\[
 y_{n+1}=W-\eta S_n,
 \qquad
 y_{n+1}\le-\mu
 \ \Longleftrightarrow\
 S_n\ge A_T,
\]
and
\[
 A_T
 =\frac{T}{\kappa M_2^{1/2}}
 \left(1+\frac1{2M_2}\right)
 \le\frac{3T}{2\kappa M_2^{1/2}}.
\]
Choose \(M_{2,0}(\kappa)\) large enough that \(M_2^{-1}\le1/4\) and
\(n_T\le T\).  Then \(\E S_n\ge n/2\), and for every \(n\ge n_T\),
\[
 A_T\le\frac{3n}{8}.
\]
Hoeffding's inequality gives
\[
 \Prb(S_n<A_T)\le e^{-n/128}.
\]
Let \(C_0=(1-e^{-1/128})^{-1}<129\).  Summing the geometric tail shows
\begin{equation}
 \Prb\!\left(y_t\le-\mu
 \text{ for every }n_T+1\le t\le T\right)
 \ge
 1-129\exp\!\left(-\frac{T}{32\kappa M_2^{1/2}}\right).
 \label{eq:flat-occupation}
\end{equation}
For \(n\ge n_T\),
\[
 \E\abs{f_{\mu,1}'(y_{n+1})}
 \le\mu\Prb(S_n<A_T)
 \le\mu e^{-n/128}.
\]
The first \(n_T\) query points have gradient magnitude at most \(\mu\).
Consequently,
\begin{align*}
 \frac1T\sum_{t=1}^T\E\abs{f_{\mu,1}'(y_t)}
 &\le
 \frac\mu T\left(n_T+C_0e^{-n_T/128}\right)\\
 &\le
 \frac4{\kappa M_2}
 +\frac1{T M_2^{1/2}}
 \left[1+129e^{-T/(32\kappa M_2^{1/2})}\right]\\
 &\le\frac5{\kappa M_2}.
\end{align*}
The last inequality uses \(M_2^{3/2}\le T\) and sufficiently large
\(M_2\).  The comparison trajectory spends only an
\(O_\kappa(M_2^{-1/2})\) fraction of the horizon outside the flat region.
\end{proof}

\section{Horizon-dependent rates}

The horizon exponent records only how the memory length is allowed to grow;
the bounds are \(M_2^{-1/2}\) for Adam and \(M_2^{-1}\) for the memoryless
update.

\begin{corollary}[Rates induced by growing memory]
\label{cor:horizon-rates}
Fix \(s\in(0,1/3]\) and \(\kappa\in(0,1/64]\).  For all sufficiently large
\(T\), set \(M_2=T^{2s}\) and use the instance in
Theorem~\ref{thm:optimization}.  Set \(\eps_t=0\) and use the
constant step \(\eta=\kappa M_2/T\) for both Adam and the memoryless update.
Then
\begin{equation}
 \frac1T\sum_{t=1}^T\E\abs{f_{\mu,1}'(x_t)}
 \ge\frac12T^{-s},
 \qquad
 \frac1T\sum_{t=1}^T\E\abs{f_{\mu,1}'(y_t)}
 \le\frac5\kappa T^{-2s}.
 \label{eq:horizon-rates}
\end{equation}
\end{corollary}

\begin{proof}
Here \(M_2^{3/2}=T^{3s}\le T\), and \(M_2\) eventually exceeds the fixed
thresholds in Theorem~\ref{thm:optimization} and
Proposition~\ref{prop:memoryless-comparison}.  Substituting
\(M_2=T^{2s}\) gives \eqref{eq:horizon-rates}.
\end{proof}

The restriction \(s\le1/3\) follows from \(M_2^{3/2}\le T\): the initialization
displacement \(M_2^2/T\) must not exceed the distance scale \(M_2^{1/2}\).

\begin{corollary}[Extension to \(1<p\le2\)]
\label{cor:finite-p}
For every \(1<p\le2\), the oracle in \eqref{eq:corridor-oracle} satisfies
the centered moment bound
\begin{equation}
 \sup_x
 \left(\E\abs{G(x,\omega)-f_{\mu,L}'(x)}^p\right)^{1/p}
 \le\sigma.
 \label{eq:finite-p-condition}
\end{equation}
\end{corollary}

\appendix

\section{Bias correction}
\label{app:bias}

For fixed \(\beta_1,\beta_2\) and any predictable nonnegative denominator
offset \(\bar\eps_t\), define the bias-corrected Adam direction
\cite{kingma2015adam}:
\[
 \widehat u_t=
 \frac{m_t/(1-\beta_1^t)}
 {\sqrt{v_t/(1-\beta_2^t)}+\bar\eps_t}.
\]
Then
\[
 \widehat u_t=
 \frac{\sqrt{1-\beta_2^t}}{1-\beta_1^t}
 \frac{m_t}
 {\sqrt{v_t}+\bar\eps_t\sqrt{1-\beta_2^t}}.
\]
The same conditioning argument gives
\begin{equation}
 \E(\widehat u_t)_+
 \le
 A_{\beta_2}
 \left(\frac\mu\sigma+\sqrt{\frac\rho{\beta_2}}\right)
 +\beta_2^t.
 \label{eq:bias-corrected}
\end{equation}
Division by \(1-\beta_1^t\) cancels the total first-moment weight.  Division
of \(v_t\) by \(1-\beta_2^t\) contributes a factor at most one to the sum over
outlier ages.  On the no-outlier event, the corrected direction is at most
one.  Standard bias correction retains the \(O(M_2^{-1/2})\) bound in
the low-signal regime \(\mu/\sigma\le\sqrt\rho\).

\section{Finite-\texorpdfstring{\(p\)}{p} moment bound}
\label{app:finite-p}

\begin{proof}[Proof of Corollary~\ref{cor:finite-p}]
Since \(G(x,\omega)-f_{\mu,L}'(x)=\xi(\omega)\), Lyapunov's inequality and
\eqref{eq:fv-moments} give, for every \(1<p\le2\),
\[
 \sup_x\left(\E\abs{G(x,\omega)-f_{\mu,L}'(x)}^p\right)^{1/p}
 =\left(\E\abs{\xi}^p\right)^{1/p}
 \le
 \left(\E\xi^2\right)^{1/2}
 =\sigma.
\]
The oracle satisfies the moment assumption throughout the full
range \(1<p\le2\).
\end{proof}

\end{document}